\documentclass{llncs}
\usepackage{makeidx}  
\usepackage{graphicx}
\usepackage{subfigure}
\usepackage{lineno}

\usepackage[T1]{fontenc}
\usepackage[english]{babel}

\usepackage{amssymb}
\usepackage{amsmath}
\usepackage{amsfonts}

\usepackage{epstopdf}

\usepackage[ruled,vlined,linesnumbered,english]{algorithm2e}

\begin{document}

\title{Watersheds on edge or node weighted graphs "par l'exemple"}
\titlerunning{Watersheds on edge or node weighted graphs}

\author{Fernand Meyer}
\authorrunning{F. Meyer} 
%
\tocauthor{Fernand Meyer (CMM)}
\institute{CMM-Centre de Morphologie Math\'{e}matique,
\\ Math\'{e}matiques et Syst\`{e}mes, MINES ParisTech, France
\email{fernand.meyer@mines-paristech.fr} }

\maketitle              

\begin{abstract}
Watersheds have been defined both for node and edge weighted graphs. We show
that they are identical: for each edge (resp.\ node) weighted graph exists a
node (resp. edge) weighted graph with the same minima and catchment basin.\ 
\end{abstract}

\section{Introduction}

The watershed is a versatile and powerful segmentation tool.\ Its use for
segmentation is due to Ch.Lantu\'{e}joul and S.Beucher \cite{beucher79}.\ It
may be applied on an image considered as a topographic surface
\cite{beucher79}, \cite{meyer91}, \cite{beucher79}.\ An image may be
considered as a node weighted graph ; the nodes are the pixels of the image,
weighted by their grey tone ; the edges connect neighboring nodes and are not
weighted.\ .\ Or the watershed may be applied on an edge weighted graph, such
as a region adjacency graph \cite{waterfalls94}, where the nodes are
unweighted and represent the catchment basins, and the edges connecting
neighboring basins are weighted by the altitude of the pass point separating
two basins.\ 

In the first case, one has to find the watershed on a node weighted graphs, in
the second on an edge weighted graph \cite{Coustywshedcut}.\ Definitions and
algorithms are not the same in both worlds, although the have the same
physical inspirations.\ The rain model, where the destiny of a drop of water
falling on the surface defines the catchment basins ; a catchment basin is the
attraction zone of a minimum, i.e. the set of nodes from where a drop of water
may reach this minimum.\ Catchment basins generally overlap.The flooding model
where the relief is flooded from sources placed at the regional minima and
meet for forming a partition. The later method being often implemented as
shortest distance algorithms \cite{meyer94}, \cite{Najman199499}. There is no
opposition between these models as the same trajectory may be followed from
bottom to top, and we have a flooding model or from top to bottom and we have
a rain model.\ A good review on the watershed may be found in
\cite{Roerdink01thewatershed}, and in the recent book \cite{mm2012}.

This paper aims at showing the equivalence between edge or node weighted
graphs for the construction of the watershed.

\section{Graphs}

\subsection{General definitions}

A \textit{non oriented graph} $G=\left[  N,E\right]  $ is a collection $N$ of
vertices or nodes and of edges $E,$ an edge $u\in E$ being a pair of vertices
(see \cite{berge85},\cite{gondranminoux}).

\textit{A chain} of length $n$ is a sequence of $n$ edges $L=\left\{
e_{1},e_{2,}\ldots,e_{n}\right\}  $, such that each edge $e_{i}$ of the
sequence $\left(  2\leq i\leq n-1\right)  $ shares one extremity with the edge
$e_{i-1}$ $(e_{i-1}\neq e_{i})$, and the other extremity with $e_{i+1}$
$(e_{i+1}\neq e_{i})$.

A \textit{path }between two nodes $x$ and $y$ is a sequence of nodes
$(n_{1}=x,n_{2},...,n_{k}=y)$ such that two successive nodes $n_{i}$ and
$n_{i+1}$ are linked by an edge.\ 

\textit{A cycle} is a chain or a path whose extremities coincide.

\textit{A cocycle} is the set of all edges with one extremity in a subset $Y$
and the other in the complementary set $\overline{Y}.$

The subgraph spanning a set $A\subset N$ is the graph $G_{A}=[A,E_{A}]$, where
$E_{A}$ are the edges linking two nodes of $A.$

The partial graph associated to the edges $E^{\prime}\subset E$ is $G^{\prime
}=[N,E^{\prime}].$

\textit{A connected graph }is a graph where each pair of nodes is connected by
a path.

\subsection{Weighted graphs: regional minima and catchment basins}

In a graph $G=\left[  N,E\right]  ,$ edges and nodes may be weighted :
$e_{ij}$ is the weight of the edge $(i,j)$ and $n_{i}$ the weight of the node
$i.$ The weights take their value in the completely ordered lattice
$\mathcal{T}$.

\subsubsection{Edge weighted graphs}

\paragraph{Regional minima}

A subgraph $G^{\prime}$ of an edge weighted graph $G$ is a flat zone, if any
two nodes of $G^{\prime}$ are connected by a chain of uniform altitude.

A subgraph $G^{\prime}$ of a graph $G$ is a regional minimum if $G^{\prime}$
is a flat zone and all edges in its cocycle have a higher altitude.

\paragraph{Catchment basins}

A chain $L=\left\{  e_{1},e_{2,}\ldots,e_{n}\right\}  $ is a \textit{flooding
chain}, if each edge $e_{k}=(n_{k},n_{k+1})$ is one of the lowest edges of its
extremity $n_{k},$ and if along the chain the weigths of the edges is never increasing.\ 

\begin{definition}
The catchment basin of a minimum $m$ is the set of nodes linked by a flooding
chain with a node within $m.\ $
\end{definition}

\subsubsection{Node weighted graphs}

\paragraph{Regional minima}

A subgraph $G^{\prime}$ of a node weighted graph $G$ is a flat zone, if any
two nodes of $G^{\prime}$ are connected by a path along which all nodes have
the same altitude.

A subgraph $G^{\prime}$ of a graph $G$ is a regional minimum if $G^{\prime}$
is a flat zone and all neighboring nodes have a higher altitude.

\paragraph{Catchment basins}

A \textit{flooding path} between two nodes is a path along which the weigths
of the nodes is never increasing.\ 

\begin{definition}
The catchment basin of a minimum $m$ is the set of nodes linked by a non
ascending path with a node within $m.\ $
\end{definition}

\begin{remark}
The definition of the catchment basins is the loosest possible, compatible
with the physical inspiration of the rain model : a drop of water falling on a
surface cannot go upwards.\ With this definition, the same node may belong to
various catchment basins.\ In other words there are large overlapping zones of
the catchment basins. As most algorithm aim at producing a partition, they
propose various methods for suppressing these overlapping zones.
\end{remark}

\section{Outline of the method}

We want to show the equivalence of the watershed on node and edge weighted
graphs.\ We first present the outline of the method on the two simple graphs
in fig.\ref{wshdex1}, the left one being edge weighted and the right being
node weighted.\ %

\begin{figure}
[ptb]
\begin{center}
\includegraphics[
height=1.2296in,
width=4.6583in
]%
{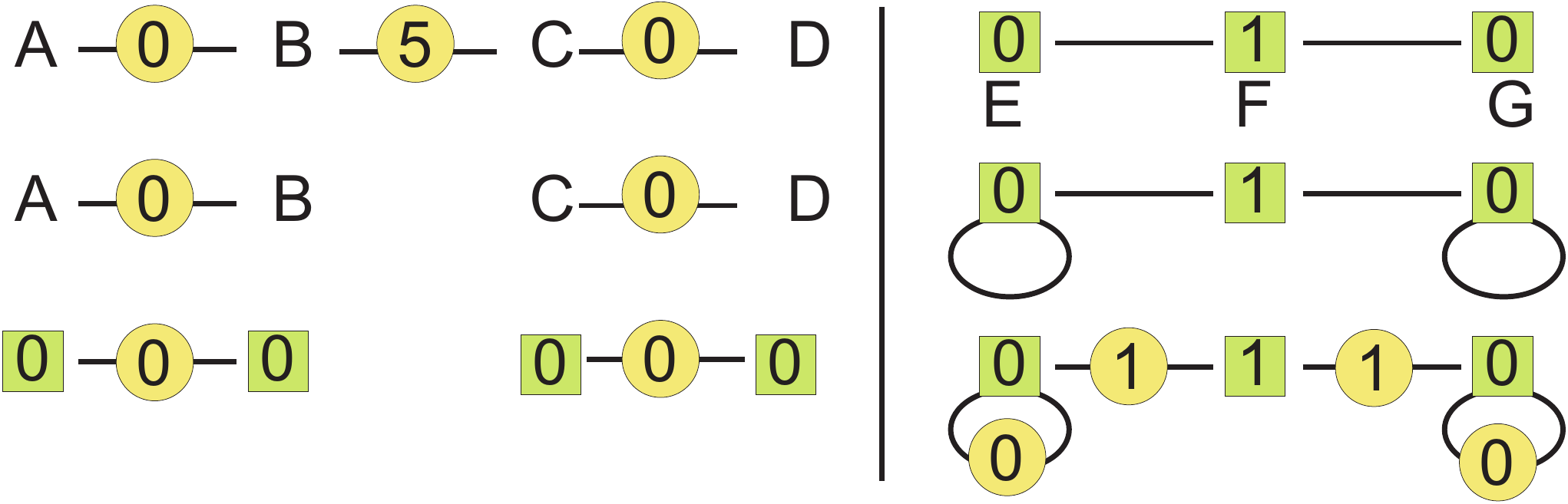}%
\caption{On the left an edge weighted graph transformed into a flooding graph.
On the right a node weighted graph transformed into a flooding graph.\ }%
\label{wshdex1}%
\end{center}
\end{figure}

Consider first the edge weighted graph. It has 4 nodes A,B,C\ and \ D
separated by weighted edges.\ In a flooding chain, each node is linked with
the next node by one of its lowest edges.\ For this reason, if an edge is not
the lowest edge of one of its extremities, it will never be crossed by a
flooding chain. It is the case for the edge BC: the lowest adjacent edge of B
is AB and the lowest adjacent edge of \ C is CD. For this reason, this edge BC
can be suppressed from the graph (as presented in the second line), without
modifying the flooding chains of the original graph.\ In a last step we assign
weights to the nodes: each node gets the weight of its lowest adjacent edge,
as represented in the third line\ .he resulting graph is called flooding
graph. In our case, it has two regional minima, which are identical if one
considers them from the point of view of the edge weights or the node
weights.\ For this simple graph, they constitute the catchment basins.\ 

Consider now the node weigthed graph on the right of fig.\ref{wshdex1}\ . It
has two isolated regional minima. One adds a loop edge linking each isolated
node with itself as shown in fig.\ref{wshdex1}.\ This modification does not
change the flooding paths of the initial graph. The last step consists in
assigning, as weight to each edge, the maximal weight of its extremities. The
added loops algo get an edge weight.\ As a result we get again a flooding
graph with the following features:

\begin{itemize}
\item the edges spanning the regional minima of the node weighted graph are
the regional minima of the edge weighted graph.\ 

\item the lowest adjacent edge of a node has the same weight as this
node.\ For this reason each flooding path of the node weighted graph is
simultaneously a flooding chain of the edge weighted graph.\ 
\end{itemize}

Having an identity between the minima and between flooding paths and chains,
the catchment basins of both graphs are the same. In our case we have two
basins with an overlapping zone containing the node F$.$

\section{The flooding graph}

\subsection{The flooding adjunction}

We define two operators between edges and nodes :\newline- an erosion $\left[
\varepsilon_{en}n\right]  _{ij}=n_{i}\wedge n_{j}$ and its adjunct dilation
$\left[  \delta_{ne}e\right]  _{i}=\underset{(k\text{\thinspace
neighbors\thinspace of\thinspace}\,i)}{\bigvee e_{ik}}$\newline- a dilation
$\left[  \delta_{en}n\right]  _{ij}=n_{i}\vee n_{j}$ and its adjunct erosion
$\left[  \varepsilon_{ne}e\right]  _{i}=\underset{(k\text{\thinspace
neighbors\thinspace of\thinspace}\,i)}{\bigwedge e_{ik}}$

The pairs $(\varepsilon_{ne},\delta_{en})$ and $(\varepsilon_{en},\delta
_{ne})$ are adjunct operators.\ The pairs $(\varepsilon_{ne},\delta_{ne})$ and
$(\varepsilon_{en},\delta_{en})$ dual operators.

We call the first pair \textbf{flooding adjunction }as we may give it a
physical explanation.\ Let us consider a region adjacency graph of a
topographical surface, where $n_{i}$ and $n_{j}$ represent the flood level in
the basins $i$ and $j,$ and $e_{ij}$ represents the altitude of the pass pont
between both basins.\ Then:\newline* the altitudes of the nodes $i$ and $j,$
the lowest flood covering $i$ and $j$ has the altitude $\left[  \delta
_{en}n\right]  _{ij}=n_{i}\vee n_{j}$\newline* if $i$ represents a catchment
basin, $e_{ik}$ the altitude of the pass points with the neighboring basin
$k,$ then the highest level of flooding without overflow through an adjacent
edge is $\left[  \varepsilon_{ne}e\right]  _{i}=\underset{(k\text{\thinspace
neighbors\thinspace of\thinspace}\,i)}{\bigwedge e_{ik}}.$

As $\varepsilon_{ne}$ and $\delta_{en}$ are adjunct operators, the operator
$\varphi_{n}=\varepsilon_{ne}\delta_{en}$ is a closing on $n$ and $\gamma
_{e}=\delta_{en}\varepsilon_{ne}$ is an opening on $e$.\smallskip

\subsection{The opening $\gamma_{e}$}

We consider first an edge weighted graph $G_{e}$\ and study the effect of the
opening $\gamma_{e}$ on its edge weights.\ Fig.\ref{wfal22} presents from left
to right: 1) an edge weighted graph, , 2) the result of the erosion
$\varepsilon_{ne}$, 3) the subsequent dilation producing an opening. The edges
in red are those whose weight has been reduced by the opening.\ The others,
invariant by the opening\textbf{\ }$\gamma_{e}$ are the edges which are, as we
establish below, the lowest edge of one of their extremities. Two
possibilities exist for an edge $(i,j)$ with a weight $\lambda:$\newline* the
edge $(i,j)$ has lower neighboring edges at each extremity. Hence
$\varepsilon_{ne}(i)<\lambda$ and $\varepsilon_{ne}(j)<\lambda$ ; hence
$\gamma_{e}=\delta_{en}\varepsilon_{ne}(i,j)=\varepsilon_{en}(i)\vee
\varepsilon_{en}(j)<\lambda:$ the edge $(i,j)$ is not invariant by the opening
$\gamma_{e}$\newline* the edge $(i,j)$ is the lowest edge of the extremity
$i.$ Then $\varepsilon_{ne}(i)=\lambda$ and $\varepsilon_{ne}(j)\leq\lambda$ ;
hence $\gamma_{e}=\delta_{en}\varepsilon_{ne}(i,j)=\varepsilon_{en}%
(i)\vee\varepsilon_{en}(j)=\lambda:$ the edge $(i,j)$ is invariant by the
opening $\gamma_{e}$

\begin{figure}[ptb]
\begin{center}
\includegraphics[
natheight=1.990800in,
natwidth=3.164100in,
height=1.37in,
width=2.18in
]{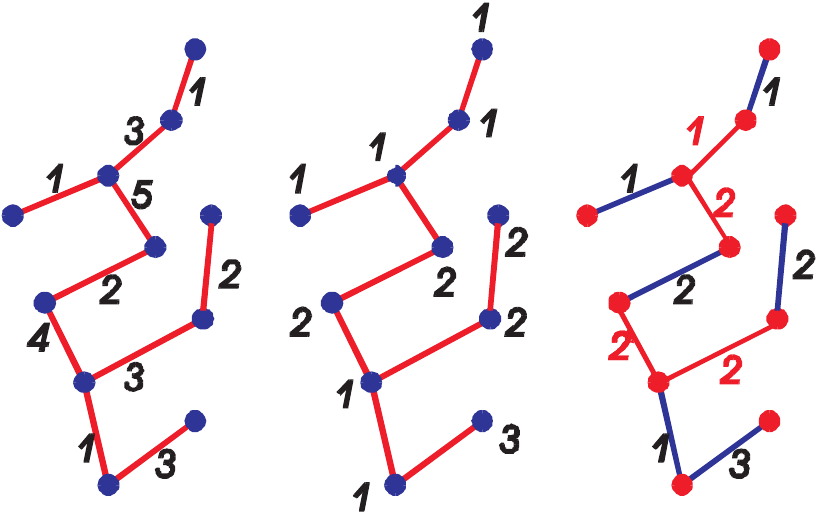}
\end{center}
\caption{From left to right: 1) an edge weighted graph, in the centre, 2) the
result of the erosion $\varepsilon_{ne}$, 3) the subsequent dilation produces
an opening. The edges in red are those whose weight has been reduced by the
opening as they were not the lowest edge of one of their extremities in the
initial distribution of edge weights. }%
\label{wfal22}%
\end{figure}

Hence the edges which are invariant by $\gamma_{e}$ are all edges which are
the lowest edges for one of their extremities. The operator keeping for each
node only its lowest adjacent edges is written $\downarrow\ :G\rightarrow
\ \downarrow G.\ $As each node has at least one lowest neighboring edge, the
resulting graph $\downarrow G$ spans all the nodes. The resulting graph also
contains all flooding chains of the initial graph (since in a flooding chain,
each edge is the lowest edge of one of its extremities).\ 

\subsubsection{The regional minima of a graph $G_{e}$ invariant by the opening
$\gamma_{e}$}

Consider an edge weighted graph $G_{e}$ invariant by the opening $\gamma_{e}.$
We assign to the nodes the weights $\varepsilon_{ne}.\ $We call $G_{n}$ the
graph on which one only considers the node weights.\ 

\begin{theorem}
If an edge weighted graph $G_{e}$ is invariant by the opening $\gamma_{e},$
then its regional minima edges span the regional minima nodes of $G_{n}.$
\end{theorem}

\textbf{Proof: }A regional minimum $m$ of the graph $G_{e}$ is a plateau of
edges with altitude $\lambda$, with all adjacent edges in the cocycle having a
weight higher than $\lambda.\ $If a node $i$ belongs to this regional minimum,
its adjacent edges have a weight $\geq\lambda$ but it has at least one
neighboring edge with weight $\lambda$ : hence the weight of $i$ is $\left(
\varepsilon_{ne}e\right)  _{i}=\lambda.$ Consider now an edge $(s,t),$ with
$s$ inside the regional minimum and $t$ outside.\ Then $e_{st}>\lambda$.\ As
$G\ $is invariant by $\gamma_{e},$ the edge $(s,t)$ is one of the lowest edges
of the nodes $t$ : thus the weight of $t$ is $\left(  \varepsilon
_{ne}e\right)  _{t}=e_{st}>\lambda.\ $This shows that the nodes spanned by the
regional minimum $m$ form a regional minimum of the graph $G_{n}.$

\subsection{The closing $\varphi_{n}$}

Consider now a node weighted graph $G.\ $The closing $\varphi_{n}$ is obtained
by a dilation $\delta_{en}$ of the node weights followed by an erosion
$\varepsilon_{ne}.\ $

\begin{lemma}
The closing $\varphi_{n}$ replaces each isolated node constituting a regional
minimum by its lowest neighboring node and leaves all other nodes unchanged.
\end{lemma}

\begin{proof}
\ Consider a node with a weight $\lambda$ belonging to a regional
minimum:\newline* consider the case where the node $i$ is an isolated regional
minimum.\ Then $\delta_{en}$ assigns to all edges adjacent to $i$ a weight
bigger than $\lambda.\ $The subsequent erosion $\varepsilon_{ne}$ assigns to
$i$ the smallest of these weights, which is the weight of the smallest
neighbor.\ The node $i$ is not invariant by the closing $\varphi_{n}.$%
\newline* Suppose that $i$ belongs to a regional minimum which is not
isolated. The dilation $\delta_{en}$ assigns to each edge adjacent to $i$ a
weight $\geq\lambda.\ $If $i$ has a a neighbor $j$ with a weight $\mu
\leq\lambda,.$then $\delta_{en}$ assigns to the edge $(i,j)$ the weight
$\lambda.\ $The subsequent erosion $\varepsilon_{ne}$ assigns to $i$ the
smallest of these weights, that is $\lambda.$ The node $i$ is invariant by the
closing $\varphi_{n}.$
\end{proof}

Hence if $G$ has isolated regional minima, it is not invariant by $\varphi
_{n}.\ $If we add a loop edge linking each isolated regional minimum with
itself we obtain a graph invariant by $\varphi_{n}.\ $Indeed, if $i$ is an
isolated regional minimum with weight $\lambda,$ we add a loop edge $(i,i)$ ;
the dilation $\delta_{en}$ assigns to the loop the weight $\lambda$ and to all
other edges adjacent to $i,$ a weight $>\lambda.\ $The subsequent erosion
$\varepsilon_{ne}$ assigns to $i$ the smallest of these weights, which is the
weight of the loop, i.e.\ $\lambda.$

We write $\circlearrowright:G\rightarrow\ \circlearrowright G$ the operator
which adds to a node weighted graph a loop between each isolated regional
minimum and itself.\ 

\subsubsection{The regional minima of a graph $G_{n}$ invariant by the closing
$\varphi_{n}$}

Consider a node weighted graph $G_{n}$ invariant by the closing $\varphi_{n}.$
We assign to the edges the weights $\delta_{en}.\ $We call $G_{e}$ the graph
on which one only consider the edge weights.\ 

\begin{theorem}
If $G$ is invariant by the closing $\varphi_{n},$ then the edges spanning the
regional minima nodes of $G_{n}$ form the regional minima edges of $G_{e}.$
\end{theorem}

\textbf{Proof:} A regional minimum $m$ of $G_{n}$ is a plateau of pixels with
altitude $\lambda$, containing at least two nodes (there are no isolated
regional minima as $G_{n}$ is invariant by $\varphi_{n}$).\ All internal edges
of the plateau get the valuation $\lambda$ by $\delta_{en}n.$ If an edge
$(i,j)$ has the extremity $i$ in the minimum and the extremity $j$ outside,
then $\delta_{en}n(i,j)>\lambda.\ $Hence, for the graph $G_{e},$ the edges
spanning the nodes of $m$ form a regional minimum.\ 

\subsection{The flooding graph}

We consider now a graph $\ G$ on which both nodes and edges are weighted.\ If
we consider only the edge weights we write $G_{e}$ and $G_{n}$ if we consider
only the node weights.\ 

\textbf{Definition: }An edge and node weighted graph $G=[N,E]$ is a flooding
graph iff its weight distribution $(n,e)$ verifies both $\delta_{en}n=e$ and
$\varepsilon_{ne}e=n.$

In a flooding graph, the weight distribution $(n,e)$ verifies $e=\delta
_{en}n=\delta_{en}\varepsilon_{ne}e\newline=\gamma_{e}e$ and $n=\varepsilon
_{ne}e=\varepsilon_{ne}\delta_{en}n=\varphi_{n}n$ showing that $n\in
\operatorname*{Inv}(\varphi_{n})$ and $e\in\operatorname*{Inv}(\gamma_{e}).$

As $G$ is invariant by $\gamma_{e},$ all its edges are the lowest edge of one
of their extremities. And as $G$ is invariant by $\varphi_{n},$ it has no
isolated regional minimum.

We have established earlier that:

\begin{itemize}
\item if a graph $G_{e}$ is invariant by $\gamma_{e},$ then its regional
minima edges span the regional minima nodes of $G_{n}.\ $

\item if a graph $G_{n}$ is invariant by $\varphi_{n},$ then the edges
spanning the regional minima nodes of $G_{n}$ form the regional minima edges
of $G_{e}.$
\end{itemize}

As a flooding graph is both invariant by $\gamma_{e}$ and by $\varphi_{n},$ we
get the following theorem.

\begin{theorem}
\textbf{ }If $G$ is a flooding graph then the node weighted graph $G_{n}$ and
the edge weighted graph $G_{e}$ have the same regional minima subgraph. More
precisely, the regional minima nodes of $G_{n}$ are spanned by the regional
minima edges of $G_{e}.$
\end{theorem}

Fig.\ref{wfal28} presents the same flooding graph, on the left with its edge
weights and on the right with its node weights : they have exactly the same
regional minima.\ \begin{figure}[ptb]
\begin{center}
\includegraphics[
height=1.3361in,
width=2.3359in
]{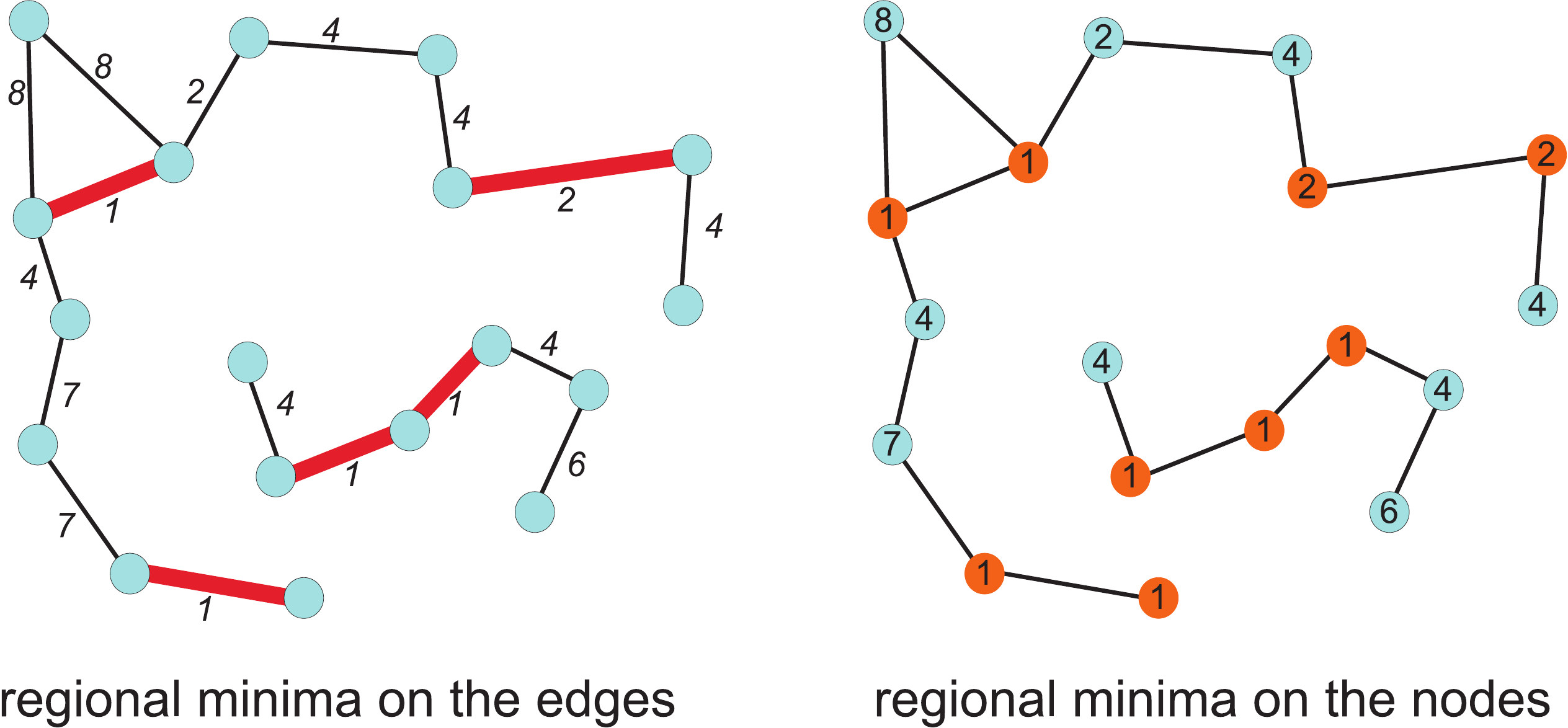}
\end{center}
\caption{Whether one consider the egde weights or the node weights produces
the same regional minima.\ }%
\label{wfal28}%
\end{figure}

\subsubsection{Flooding paths and flooding chains}

Consider a flooding path on $G_{n},$ i.e. a path $(n_{1},n_{2},...n_{k})$ of
neighboring nodes with a non increasing weight, starting at node $n_{1}$ and
ending at node $n_{k}$ belonging to a regional minimum of $G_{n}$. As
$n_{i}\geq n_{i+1}$, the weight of the edge $\left(  n_{i},n_{i+1}\right)  $
obtained by $\delta_{en}$ is equal to the weight of $n_{i}$ and is one of teh
lowest edges of $n_{i}.$ This shows, that the series of edges $\left(
n_{i},n_{i+1}\right)  $ form a flooding chain of $G_{e}$, ending in a regional
minimum of $G_{e}.\ $

Inversely, consider a flooding chain$[(n_{1},n_{2}),(n_{2},n_{3}%
),...(n_{k-1},n_{k})]$ of $G_{e}.$ The weights of the edges are not
increasing.\ Furthermore, the edge $\left(  n_{i},n_{i+1}\right)  $ is the
lowest edge of the node $n_{i}$.\ As in a flooding graph the lowest edge of a
node has the same weight than this node, the edge $\left(  n_{i}%
,n_{i+1}\right)  $ has the same weight than the node $n_{i}.\ $Thus the path
$(n_{1},n_{2},...n_{k})$ also is a flooding path of $G_{n}$ ending a regional
minimum of $G_{n}.$

There is a one to one correspondance between the regional minima of $G_{e}$
and $G_{n}$ ; there is also a one to one correspondance between the flooding
paths and the flooding tracks, with the same weight distribution.\ This shows
that both graphs, $G_{e}$ on which we only consider the edge weights and
$G_{n}$ on which we only consider the node weights, have the same catchment basins.

\subsubsection{Transforming an edge weighted graph into a flooding graph}

Consider an arbitrary edge weighted graph $G_{e}$ with edge weights $e,$ and
without node weights.\ The operator $\downarrow G$ suppresses all edges which
are not invariant by $\gamma_{e}.\ $The remaining edges are invariant by
$\gamma_{e}$ and verify : $e=\delta_{en}\varepsilon_{ne}e.\ $

If we assign to the nodes of $G_{e}$ weights equal to $n=\varepsilon_{ne}e,$
then $e=\delta_{en}\varepsilon_{ne}e=\delta_{en}n,$ and the resulting graph is
a flooding graph.\ 

\paragraph{Illustration}

Fig.\ref{wshdex4}A presents an edge weighted graph $G,$ on which the node
weights are produced by the erosion $\varepsilon_{ne}\ $In fig.\ref{wshdex4}B
a dilation $\delta_{en}$ applied after the erosion produces an opening of the
initial edge weights. The weights with a red color are those which are lowered
by the opening. These edges are suppressed producing the graph $\downarrow G$
in fig.\ref{wshdex4}C.\ The erosion $\varepsilon_{ne}e$ applied on $G$ or on
$\downarrow G$ produces the same node weights.\ The resulting graph is a
flooding graph.

Fig.\ \ref{wshdex4}D shows the regional minima of the complete graph if one
considers only the node weights.\ In contrast, fig.\ \ref{wshdex4}E shows the
regional minima of the flooding graph, if one only considers the node
weights.\ They are not identical.\ 

If we consider only the edge weights of the flooding graph, we write $G_{e}$
and $G_{n}$ if we consider only the node weights.\ One verifies that the edges
spanned by the regional minima nodes of $G_{n}$ in fig.\ \ref{wshdex4}E are
spanned by the regional minima edges of $G_{e}$ in fig.\ref{wshdex4}F. One
also verifies that flooding paths and flooding chains are identical, each node
being followed by an edge with the same weight.\
\begin{figure}
[ptb]
\begin{center}
\includegraphics[
height=2.6036in,
width=4.6534in
]%
{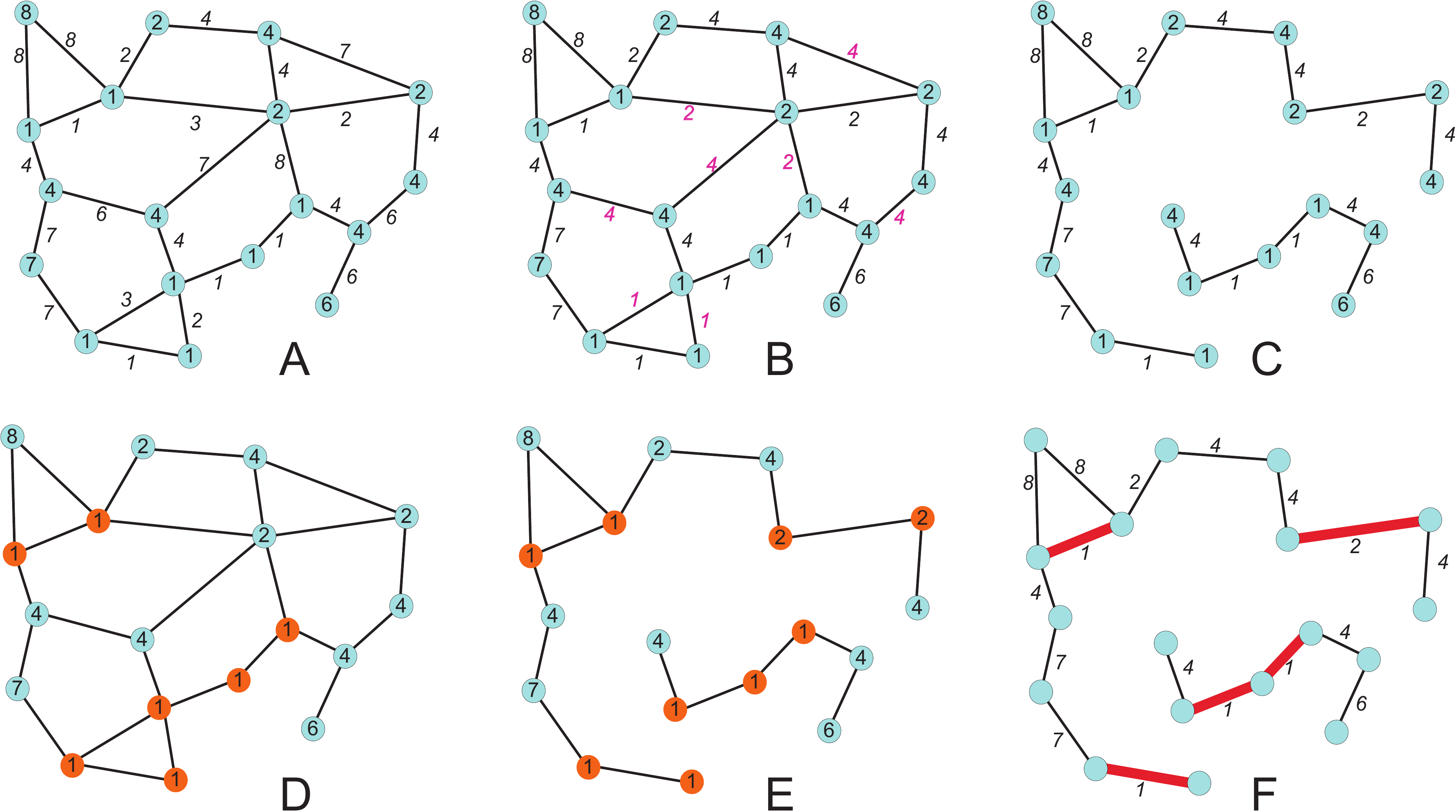}%
\caption{A: Given an edge weighted graph, the node weights are obtained by the
erosioin $\varepsilon_{ne}e.$\newline B: A subsequent dilation $\delta_{en}$
produces the opening $\gamma_{e}e=\delta_{en}\varepsilon_{ne}e.\ $The edges
which are not invariant by the opening are in red.\newline C: The edges which
are not invariant by the opening $\gamma_{e}$ are the edges which are not the
lowest adjacent edge of one of their extremities. The suppression of all these
edges produces the flooding graph.\newline D: The regional minima of the
complete node weighted graph do not correspond to the regional minima of the
edge weighted graph depicted in F.\newline E: The regional minima of node
weighted flooding graph do correspond to the regional minima of the edge
weighted graph depicted in F.\newline F: The regional minima of the edge
weighted flooding graph.}%
\label{wshdex4}%
\end{center}
\end{figure}

\subsubsection{Transforming a node weighted graph into a flooding graph}

Consider an arbitrary node weighted graph $G_{n}$with node weights $n,$ and
without edge weights.\ The operator $\circlearrowright G$ adds a loop edge
between each isolated regional minimum and itself, producing a graph invariant
by $\varphi_{n}.\ $The nodes verify $n=\varepsilon_{ne}\delta_{en}n.$

If we assign to the edges of $G_{n}$ weights equal to $e=\delta_{en}n,$ then
$n=\varepsilon_{ne}\delta_{en}n=\varepsilon_{ne}e$ and the resulting graph is
a flooding graph.\ 

\paragraph{Illustration}

Fig.\ref{wshdex7}A presents a node weighted graph.\ It is not invariant by
$\varphi_{n}$ as it has isolated regional minima.\ One adds a loop edge
linking each isolated regional minimum with itself, producing the graph in
fig.\ref{wshdex7}B. The dilation $\delta_{en}$ produces the edge weights.\ The
resulting edge and node weighted graph in \ref{wshdex7}C is a flooding graph.
The regional minima have been highlighted by distinct colors. Again, the
identity between flooding paths and flooding chains is clearly visible.%

\begin{figure}
[ptb]
\begin{center}
\includegraphics[
height=2.0822in,
width=4.6534in
]%
{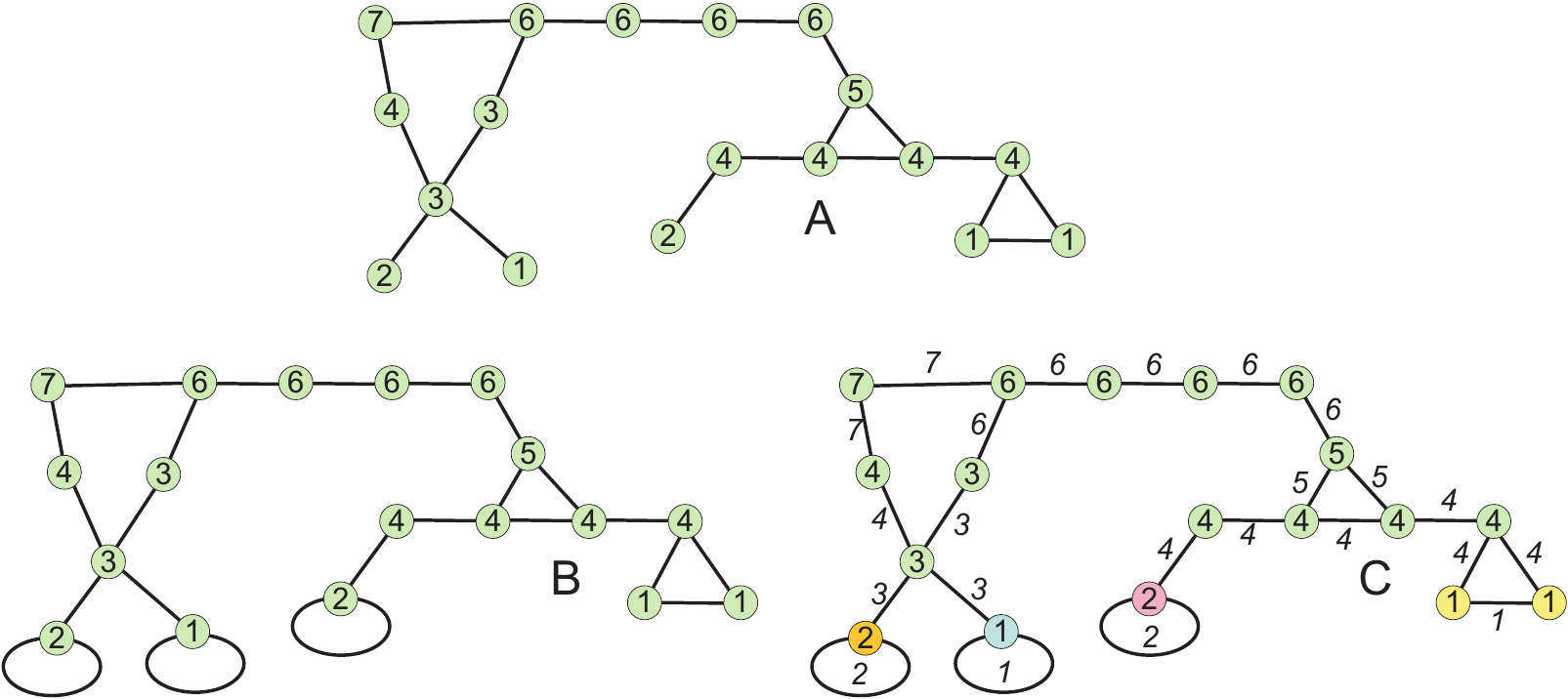}%
\caption{A: A node weighted graph.\newline B: A loop edge is added, linking
each isolated regional minimum with itself.\newline C: The edge weights are
produced by the dilation $\delta_{en}.\ $The resulting graph is a flooding
graph.\ }%
\label{wshdex7}%
\end{center}
\end{figure}

\subsection{Temporary conclusion}

We have presented how to derive from any node or edge weighted graph a
flooding graph with the same regional minima.\ To each flooding path
corresponds a flooding chain with the same weight distribution, as each node
is followed by an edge with the same weight.\ The reverse also is true: to
each flooding chain corresonds a flooding path with the same weight distribution.\ 

The catchment basins, as defined earlier, are exactly the same.\ As all non
ascending paths and chains are accepted for defining the catchment basins, we
also have the largest overlaping zones between them.

Two classes of watershed algorithms have been developed, the ones for node
weighted graphs, the others for edge weighted graphs. Since we have a one to
one correspondance between flooding paths and flooding chains, each algorithm
developed for a node (resp. edge) weighted graph may now also be applied for
an edge (.resp node) weighted graph, by applying it on the associated flooding graph.

\section{Reducing the number of catchment basins and the size of their
overlapping zones}

\subsection{Estimating the number of catchment basins}

Fig.\ref{wshdex8}A1 presents a flooding graph with three regional minima.\ The
three associated catchment basins overlap at a node p with weight 4.\ In order
to break the tie, only one edge adjacent to the node $p$ should be kept, and
the others suppressed.\ Three solutions are possible, illustrated by the
figures \ref{wshdex8}A2, B2 and C2.\ There is no means to decide between one
or the other solution if one considers only the weight on one edge.\ If one
considers the first two edges of each flooding path, we obtain a lexicographic
measure, by concatenating the weight of the first and the second, illustrated
in fig.\ \ref{wshdex8}B1.\ There remains now two choices between the edges
with weights $43,$ represented in figures figures \ref{wshdex8}B2 and
C2.\ Considering a third edge along the flooding paths leaves only one choice,
as the edge with the weight $431$ is the lowest.\ The corresponding waterhed
segmentation is presented in fig. \ref{wshdex8}C2. This example shows that we
are able to reduce the number of partitions associated to a flooding graph, if
one considers not only the first neighboring nodes or edges in the flooding
paths or chains, but a number of nodes, ordered in a lexicographic order.\ %

\begin{figure}
[ptb]
\begin{center}
\includegraphics[
height=2.8709in,
width=3.5118in
]%
{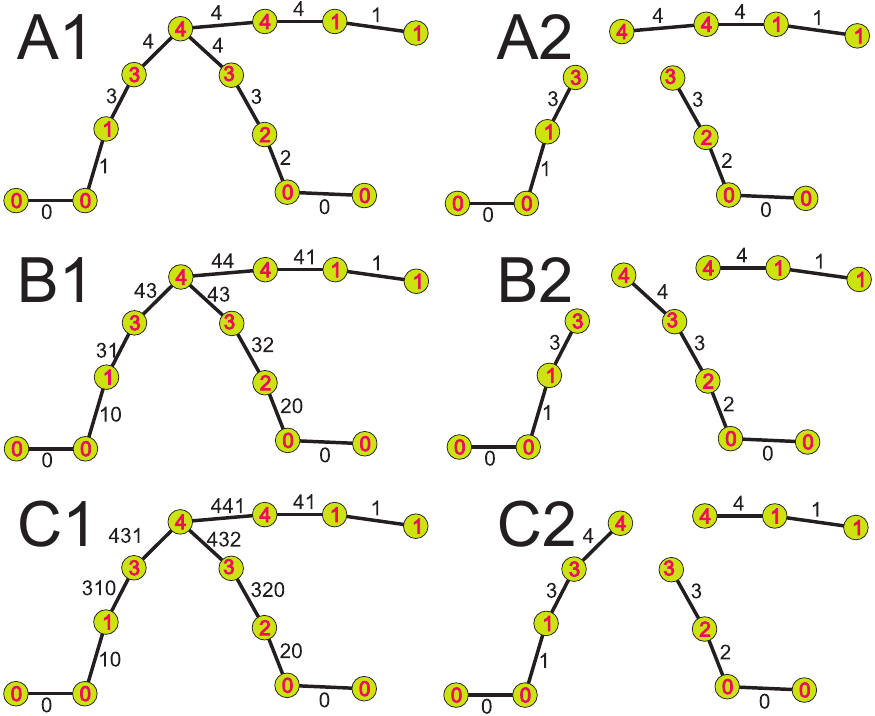}%
\caption{A1 : To each edge is assigned its lexicographic distance of depth $1$
to the nearest regional minimum. A node with weight $4$ has three adjacent
edges with weight $4.\ $Keeping only one of them and suppressing the others
yields 3 possible partitions represented in A2, B2 and C2.\newline B1: To each
edge is assigned its lexicographic distance of depth $2$ to the nearest
regional minimum. The node with weight $4$ has now two lowest adjacent edges
with a weight $43.\ $Keeping only one of them and suppressing the other
adjacent edges yields 2 possible partitions represented in B2 and C2.\newline
C1: To each edge is assigned its lexicographic distance of depth $3$ to the
nearest regional minimum. The node with weight $4$ has only one lowest
adjacent edges with a weight $431.\ $Keeping only this one of them and
suppressing the other adjacent edges yields 1 possible partitions represented
in C2.}%
\label{wshdex8}%
\end{center}
\end{figure}

By defining a lexicographic order among the flooding chains of depth $3$ has
permitted to break the ties, leaving only one solution.

\subsection{A lexicographic order relation between downwards paths}

Given a node or weighted graph we first derive its flooding graph $G=[E,N].$
We associate to $G$ an oriented graph $\overrightarrow{G}$ by replacing the
edge $(p,q)$ by an arrow $\overrightarrow{pq}$ if $n_{p}\geq n_{q}\ $and$\ $by
two arrows $\overrightarrow{pq}$ and $\overrightarrow{qp}$ if $n_{p}=n_{q}$.
The loop edge linking an isolated regional minimum node $m$ with itself also
is replaced by an arrow $\overrightarrow{mm}.$ The graph $\overrightarrow{G}$
verifies the property $(P)$: $\forall p\in N,$ there exists at least an
oriented path $\overrightarrow{\pi}$ of $\overrightarrow{G}$ (with a positive
or null length) linking $p$ with a regional minimum.\ We define the catchment
basin of a regional minimum as the set of nodes linked by an oriented path
with this minimum.\ Obviously, each node belongs to at least one catchment
basins. Catchment basins may overlap and form a watershed zone when two paths
having the same node as origin reach two distinct regional minima. We aim at
pruning the graph $\overrightarrow{G},$ without any arbitrary choices, and get
a partial graph $\overrightarrow{G^{\prime}}$ for which the property $(P)$
still holds but the watershed zones are smaller.

As soon the path $\overrightarrow{\pi}$ reaches a regional minimum, it may be
prolonged into a path of infinite length, by infinitely cycling between $2$
nodes within the regional minimum or along the loop joining each isolated
regional minimum with itself.\ All oriented paths or chains are thus of
infinite length.\ And we may consider them, either in their full infinite
length or consider only the first $k$ edges.\ 

We now define a family of preorder relations (order relation without
antisymmetry) between the paths of $\overrightarrow{G}.$

The lexicographic preorder relation of length $k$ compares the infinite paths
$\pi=(p_{1},p_{2},...p_{k},...)$ and $\chi=(q_{1},q_{2},...q_{k},...)$ by
considering the $k$ first nodes and edges:\newline* $\pi\prec^{k}\chi$ if
$n_{p_{1}}<n_{q1}$ or there exists $t<k$ such that
\begin{tabular}
[c]{l}%
$\forall l<t:n_{p_{l}}=n_{q_{l}}$\\
$n_{p_{t}}<n_{q_{t}}$%
\end{tabular}
\newline* $\pi\preceq^{k}\chi$ if $\pi\prec^{k}\chi$ or if $\forall l\leq
k:n_{p_{l}}=n_{q_{l}}.$

This preorder relation is total, as it permits to compare all paths ; for this
reason, among all paths linking a node $p$ with a regional minimum, there
exists always at least one which is the smallest for $\preceq^{k}.$ We say
that this path is the steepest for the lexicographic order of depth $k.$

For $k=\infty,$ we consider the infinite paths and we simply write $\preceq$.

If $\pi$ and$\chi$ are two paths of infinite length verifying $\pi\preceq\chi$
, then the paths $\pi_{l}$ and$\chi_{l}$ obtained by skipping the $l$ first
nodes also verify $\pi_{l}\preceq\chi_{l}.\ $If $\pi$ is the smallest path
linking its origin with a regional minimum, then $\pi_{l}$ is the smallest
path leading from $p_{l+1}$ to the same regional minimum.\ 

\subsubsection{Nested catchment basins}

Consider two lexicographic order relations $\prec^{k}$ and $\prec^{l}$ with
$l>k,$ then for $\pi_{1}$ and $\pi_{2}:$ $\pi_{1}\prec^{k}\pi_{2}%
\Rightarrow\pi_{1}\prec^{l}\pi_{2}$ or equivalently $\pi_{1}\succeq^{l}\pi
_{2}\Rightarrow\pi_{1}\succeq^{k}\pi_{2}:$ the steepest path for the
lexicographic order $l$ also is steepest for the lexicographic order $k$.\ As
a consequence, a catchment basin for $\succeq^{l}$ is included in the
catchment basin for $\succeq^{k}.$

For increasing values of $k,$ the catchment basins become larger, are nested,
and the watershed zones are reduced or vanish. For $k=\infty,$ a node is
linked by two minimal paths with two distinct minima, only if these two paths
have exactly the same weights, which seldom happens in natural images. In
particular, if the regional minima have distinct weights, the catchment basins
form a partition.\ 

\subsubsection{Pruning the flooding graph to get steeper paths}

We associate to each order relation $\preceq^{k}$ of length $k$ a pruning
operator $\downarrow^{k}.\ $The pruning $\downarrow^{k}$ suppresses each edge
which is not the first edge of a steepest path for $\preceq^{k}$ among all
paths with the same origin $.$After pruning, each node outside the regional
minima is the origin of one or several $k-$steepest flooding tracks or
$k-$steepest flooding paths\ We say that the graph $\downarrow^{k}G$ has a
k-steepness or is k-steep. As for $l>k,$ $\pi_{1}\succeq^{l}\pi_{2}%
\Rightarrow\pi_{1}\succeq^{k}\pi_{2}$, we have $\downarrow^{l}G\subset
\downarrow^{k}G$.\ Furthermore $\downarrow^{k}\downarrow^{l}G=\downarrow
^{l}\downarrow^{k}G=\downarrow^{k\vee l}G.$

\begin{remark}
Each pruning $\downarrow^{k}G$ suppresses a number of edges still present in
$\downarrow^{k-1}G$ ; it suppresses them all, without doing any arbitrary
choices between them.
\end{remark}

\paragraph{Particular k-steep graphs}

Applied to an arbitrary graph, the pruning $\downarrow^{1}=\ \downarrow$
suppresses the edges which are not the lowest edge of one of their
extremities.\ In a flooding graph, each edge it the lowest edge of one of its
extremities and $\downarrow^{1}$ is inoperant. The pruning $\downarrow^{2}$
keeps for each node $i$ the adjacent edges linking $i$ with one of its lowest
neighboring nodes. The pruning $\downarrow^{\infty}$ only keeps the first edge
of the steepest paths.

\textbf{Lemma: }Any oriented path in $\downarrow^{k}G$ of length $k$ is of
maximal steepness for $\preceq^{k}.$

For this reason, a node $p$ belongs to a $k-$catchment basin associated to a
node $m$ in a regional minimum, if there exists an oriented path in
$\downarrow^{k}G$ from $p$ to $m.$\smallskip For increasing values of $k,$ the
catchment basins are decreasing, and so are the overlapping zones between them.

\subsection{Erosions, dilations and openings on oriented graphs}

The operator $\downarrow^{k}$ defined above has nice properties but is not a
local operator. It is however possible to implement it using only local
operators as we present now.

\subsubsection{Two adjunctions on oriented graphs}

The adjunctions $(\delta_{en},\varepsilon_{ne})$ and $(\delta_{ne}%
,\varepsilon_{en})$ were defined for non oriented graphs.\ We now define the
equivalent operators for oriented graphs.\ 

The erosion from arrows to nodes assigns to each node $p$ the minimal weight
of all arrows having $p$ as origin: $\left(  \overrightarrow{\varepsilon}%
_{ne}\right)  _{p}=%
{\textstyle\bigwedge\limits_{q\mid p\rightarrow q}}
e_{\overrightarrow{pq}}.$

The dilation is obtained by adjunction.\ Consider a weight distribution $s$ on
the nodes of the oriented graph and a weight distribution $e$ on the arrows.\ 

$%
{\textstyle\bigwedge\limits_{q\mid p\rightarrow q}}
e_{\overrightarrow{pq}}\geq s_{p}\Leftrightarrow\ $for $q\mid p\rightarrow
q:e_{\overrightarrow{pq}}\geq s_{p}\Leftrightarrow e_{\overrightarrow{pq}}\geq%
{\textstyle\bigvee\limits_{q\mid p\rightarrow q}}
s_{p}=s_{p}=\left(  \overrightarrow{\delta}_{en}s\right)  _{\overrightarrow
{pq}}$

The dilation $\left(  \overrightarrow{\delta}_{en}s\right)  _{\overrightarrow
{pq}}$ assigns to the arrow $\overrightarrow{pq}$ the weight of its origin
$p.$\ 

We also will need the dual erosion $\left(  \overrightarrow{\varepsilon}%
_{en}n\right)  _{\overrightarrow{pq}}$ assigning to the arrow $\overrightarrow
{pq}$ the weight of its extremity $q$.\ 

\subsubsection{The invariants by the opening}

Consider an arrow $\overrightarrow{pq}.\ \ $The erosion $\overrightarrow
{\varepsilon}_{ne}$ assigns to the node $p$ the minimal weight of all arrows
having $p$ as origin.\ The subsequent dilation assigns to the arrow
$\overrightarrow{pq}$ the weight of $p,$ i.e. the the minimal weight of all
arrows having $p$ as origin.\ Thus if $\left(  \overrightarrow{\gamma}%
_{e}e\right)  _{\overrightarrow{pq}}$ leaves the arrow $\overrightarrow{pq}$
unchanged, it means that this arrow is one of the lowest arrows having $p$ as origin.\ 

We define a pruning operator $\downharpoonright\overrightarrow{G}$ which cuts
all arrows which are not invariant by the opening $\overrightarrow{\gamma}%
_{e},$ i.e. which are not one of the lowest edges of their origin.\ 

\subsubsection{The oriented flooding graphs}

We say that an edge and node weighted graph is an oriented flooding graph if
the weights of the nodes and of the arrows verify:

for the node weights $n_{p}=\left(  \overrightarrow{\varepsilon}_{ne}\right)
_{p}$ and for the edge weights: $e_{\overrightarrow{pq}}=\left(
\overrightarrow{\delta}_{en}n\right)  _{\overrightarrow{pq}}.$ Such a graph is
invariant by the opening $\overrightarrow{\delta}_{en}\overrightarrow
{\varepsilon}_{ne}=\overrightarrow{\gamma}_{e}$ and by the closing
$\overrightarrow{\varepsilon}_{ne}\overrightarrow{\delta}_{en}=\overrightarrow
{\varphi}_{n}.$

\subsection{Pruning a flooding graph with local operators and without
arbitrary choices.\ }

\subsubsection{The pruning operator}

We start with a node and edge weighted flooding graph $G.\ $As explained
above, we associate to $G$ an oriented graph $\overrightarrow{G}$ by replacing
the edge $(p,q)$ by an arrow $\overrightarrow{pq}$ if $n_{p}\geq n_{q}%
\ $and$\ $by two arrows $\overrightarrow{pq}$ and $\overrightarrow{qp}$ if
$n_{p}=n_{q}$. The loop edge linking an isolated regional minimum node $m$
with itself also is replaced by an arrow $\overrightarrow{mm}.$

It is easy to show that if $G$ is a flooding graph, then $\overrightarrow{G}$
is an oriented flooding graph.\ 

In order to identify edges which belong to flooding chains of maximal
lexicographic $k-$steepness, we have to shift the weight distribution of nodes
and weights along the oriented flooding track upwards. This is obtained thanks
to the erosion $\left(  \overrightarrow{\varepsilon}_{en}n\right)
_{\overrightarrow{pq}}$ which assigns to the arrow $\overrightarrow{pq}$ the
weight of its extremity $q.$

After this erosion, we get an edge weight distribution which is not invariant
anymore by the opening $\overrightarrow{\gamma}_{e}.\ $Applying the pruning
operator $\downharpoonright\overrightarrow{G}$ leaves a graph which is
invariant by $\overrightarrow{\gamma}_{e}.\ $A final erosion $\left(
\overrightarrow{\varepsilon}_{ne}\right)  _{p}$ assigns to the nodes their new weights.\ 

Thus we have applied to $\overrightarrow{G}$ the following operators
:$\overrightarrow{\varepsilon}_{en}$ from nodes to arrows, the pruning
operator $\downharpoonright$ on the edges followed by a last erosion
$\overrightarrow{\varepsilon}_{ne}$ from the edges to the nodes, producing
node weights $n=\overrightarrow{\varepsilon}_{ne}e.\ $The resulting graph is
an oriented flooding graph.\ Indeed the edges verify $e=\overrightarrow
{\gamma}_{e}e=\overrightarrow{\delta}_{en}\overrightarrow{\varepsilon}%
_{ne}e=\overrightarrow{\delta}_{en}n.\ $We call $\zeta$ the succession of
these three operators.\ In every day words, the operator $\zeta$ does the
following: each node is assigned the minimal weight of all arrows for which it
is origin ; each arrow is assigned a weight equal to its extremity ; each
arrow with a higher weight than its origin is suppressed.\ 

Every time that we apply $\zeta,$ new edges are pruned and the weight
distribution along the flooding paths and chains moves upwards.\ For
$\zeta^{(k)}$ we obtain a graph $\zeta^{(k)}\overrightarrow{G}$ which contains
only flooding paths and chains of steepness $\geq k+1.$

\subsubsection{Illustration}

\paragraph{Case of an initially edge weighted graph}

The fig. \ref{wshdex9}A represents an edge weighted graph on which the
regional minima edges are indicated in red.\ Fig. \ref{wshdex9}B presents the
associated flooding graph $G$.\ We associate to $G$ an oriented graph
$\overrightarrow{G}$ by replacing the edge $(p,q)$ by an arrow
$\overrightarrow{pq}$ if $n_{p}\geq n_{q}\ $and$\ $by two arrows
$\overrightarrow{pq}$ and $\overrightarrow{qp}$ if $n_{p}=n_{q}$ and get fig.
\ref{wshdex9}C. The operator $\zeta$ produces a new weight distribution for
both nodes and edges in fig. \ref{wshdex9}D.\ Furthermore two arrows are cut
as they are not invariant by the opening $\overrightarrow{\gamma}_{e}.$ As
each connected component contains only one regional minimum, we can stop the
pruning and we obtain the final watershed partition.\ This partition is
indicated in false color on top of the flooding graph $G$ in fig.\ref{wshdex9}%
E and also in false color on top of the complete initial graph in
fig.\ref{wshdex9}F.%
\begin{figure}
[ptb]
\begin{center}
\includegraphics[
height=2.6442in,
width=4.6542in
]%
{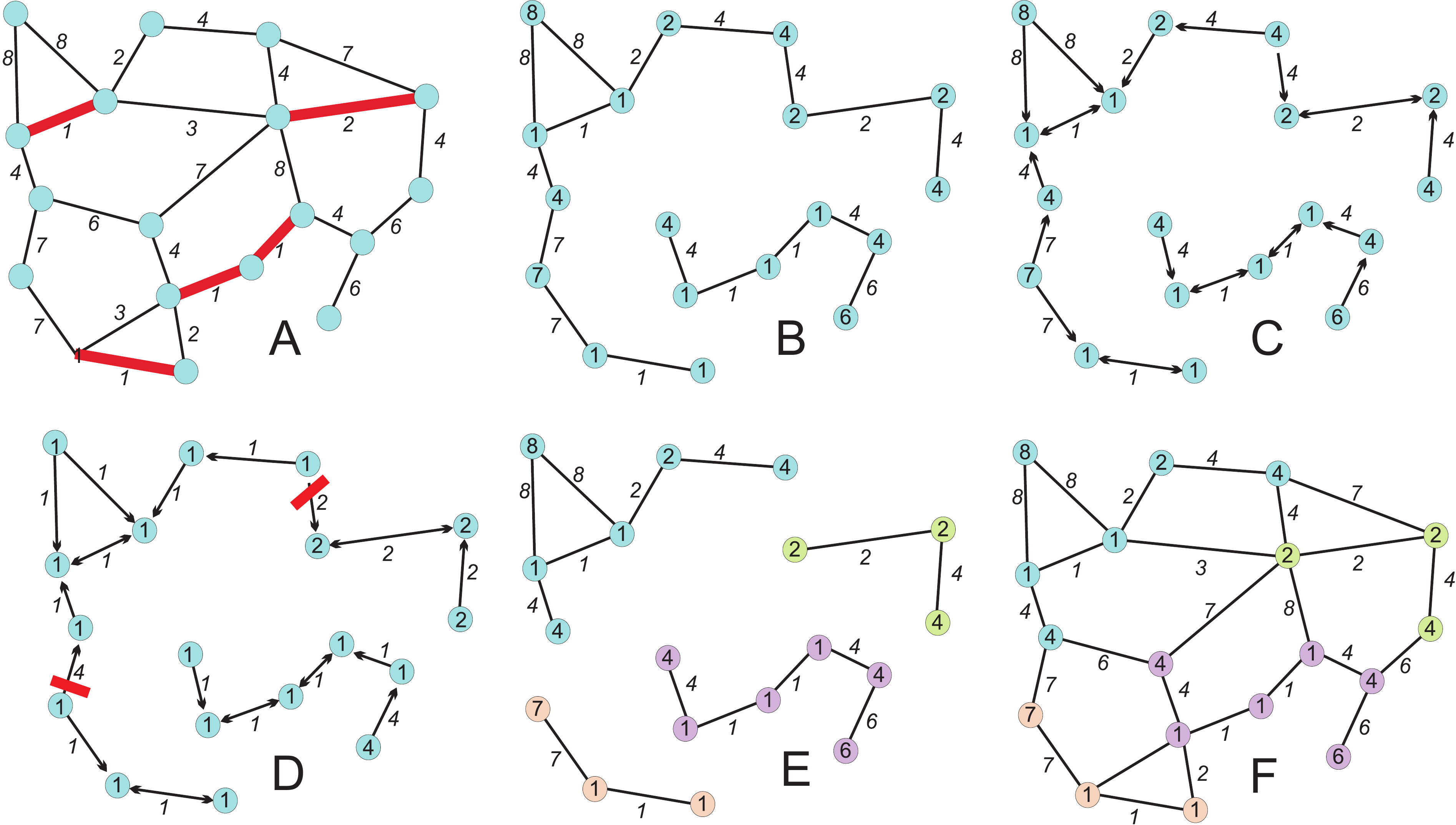}%
\caption{A: An edge weighted graph.\ The edges in red are the regional minima
edges.\newline B: The associated flooding graph $G$\newline C: The oriented
flooding graph \newline D: Applying the operator $\zeta$ to the oriented
flooding graph suppresses two edges, and leaves 4 connected components,
containing each a regional minimum.\newline E: The resulting partition
superimposed on the flooding graph.\newline F: The resulting partition
superimposed on the initial edge weighted graph}%
\label{wshdex9}%
\end{center}
\end{figure}

\paragraph{Case of an initially node weighted graph}

Fig.\ref{wshdex10}A presents a node weighted graph.\ It contains a number of
some difficulties, like the presence of 2\ plateaux and of a buttonhole at the
node with weight $3.$ In a first step we add loop edges linking each isolated
regional minimum with itself as shown in fig.\ref{wshdex10}B.\ The next step
produces edge weights by the dilation $\delta_{en}$ of the node weights,
yielding the initial flooding graph $G$ in fig.\ref{wshdex10}C The oriented
graph $\overrightarrow{G}$ is produced in fig.\ref{wshdex10}D.\ With each new
applications of the operator $\zeta$ the graph $\overrightarrow{G}$ is further
pruned, yielding the graphs of fig.\ref{wshdex10} E,F and
G.\ Fig.\ref{wshdex10}H transports the arrows of the graph fig.\ref{wshdex10}G
onto the graph $\overrightarrow{G}.\ $The remaining flooding tracks are
clearly visible, on a graph which has kept the initial node weights.

\
\begin{figure}
[ptb]
\begin{center}
\includegraphics[
height=4.2648in,
width=4.6534in
]%
{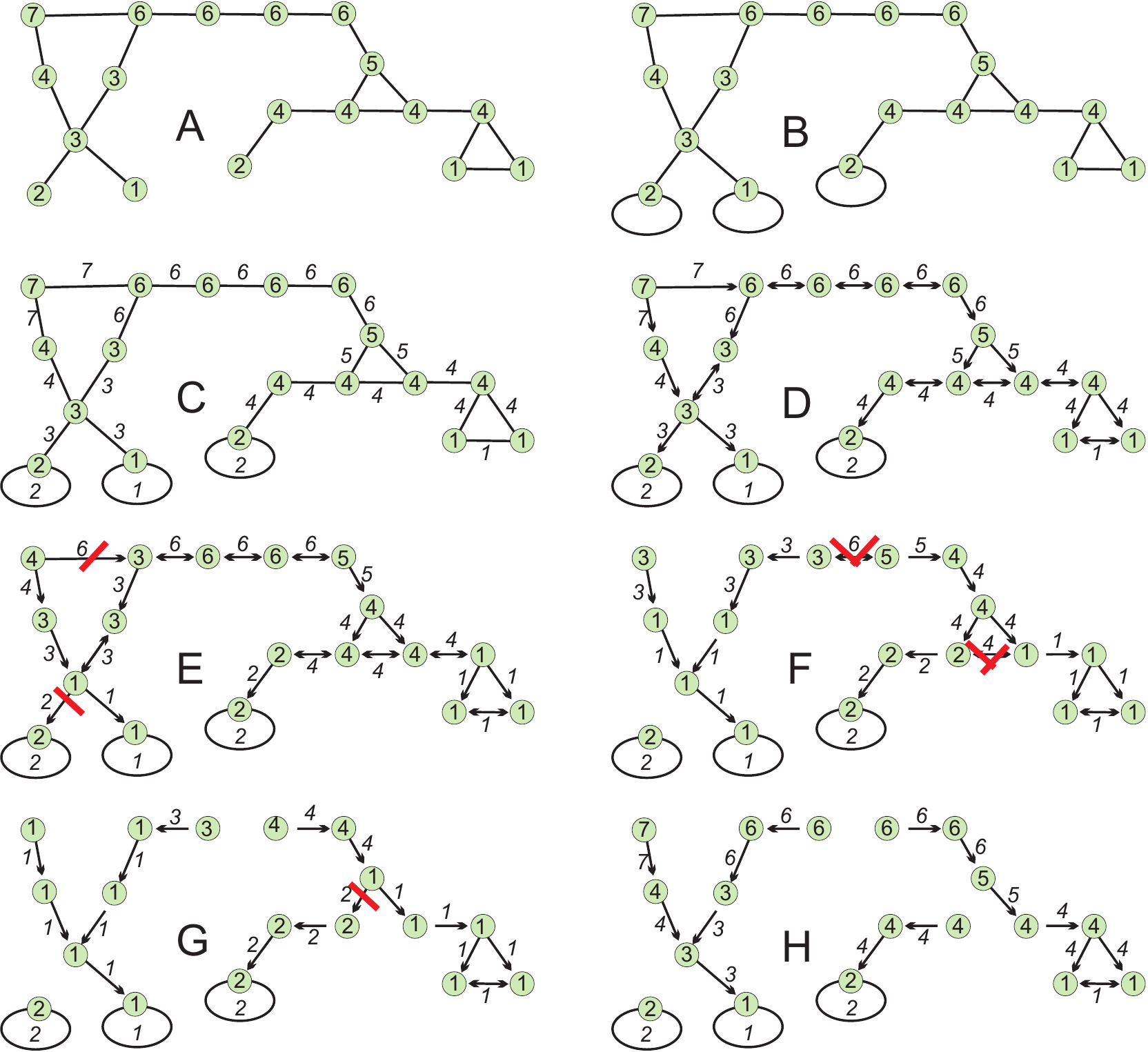}%
\caption{A: A node weighted graph.\ \newline B: A loop edge is added linking
each isolated regional minimum with itself\newline C: The dilation
$\delta_{en}$ assigns weights to the edges and produces the flooding graph
$G$\newline C: The flooding graph is transformed into an oriented flooding
graph \newline E,F,G: Applying three times the operator $\zeta$ to the
oriented flooding graph.$\ \ $Each new application suppresses new edges,
finally leaving flooding paths and chains with a lexicographic steepness of
$4.\ $It leaves 4 connected components, containing each a regional
minimum.\newline E,F,G: The resulting partition superimposed on the flooding
graph\newline F: The resulting partition superimposed with the oriented
flooding graph}%
\label{wshdex10}%
\end{center}
\end{figure}

\section{Introducing arbitrary choices when needed}

\subsection{Arbitrary choices on the flooding graph}

We now summarize the results of this study. Starting from a node or edge
weighted graph, we want to construct a watershed partition, without arbitrary
choices, or with a minimum number of arbitrary choices. In a first step we
derive the flooding graph.\ An edge weighted graph looses the edges which are
not the lowest neighboring edges of one of their extremities, but keeps the
edge weights. The node weights are obtained by the erosion $\varepsilon
_{ne}.\ $To a node weighted graph are added some loops linking each isolated
regional minimum with itself.\ The node weights keep their initial weights and
the edge weights are obtained by the dilation $\delta_{en}.$ In both cases we
obtain a graph with a perfect coupling between edge and node weights ;
furthermore, the edge regional minima span the node regional minima. And the
flooding chains span the flooding paths, each node being followed by an edge
with the same weight.\ 

\textbf{Method 1: }The ties are broken by a watershed algorithm applied on the
flooding graph.\ Any algorithm of the literature developed for node
(resp.\ edge) weighted graphs may now be applied to the flooding graph, even
if the initial graph is an edge (resp.\ node) weighted graph. The algorithm of
B.Marcotegui et al \cite{waterfallsbs} for constructing the watershed on an
edge weighted graph is derived from Prim's algorithm for constructing a
minimum spanning tree or forest.\ This algorithm is myopic and considers only
flooding chains with a lexicographic depth equal to 1. The algorithms proposed
in \cite{Coustywshedcut} are or the same myopic type.\ Using an algorithm
\cite{meyer91} based on the topographic distance \cite{meyer94}%
,\cite{Najman199499}, and applied to node weighted graphs, choses flooding
chains with a lexicographic depth equal to 2, and for this reason are more
selective.\ Such an algorithm, designed for node weighted graphs may now be
applied to a graph which initially was edge weighted.\ The arbitrary choices
for producing a watershed partition is taken in charge by the watershed
algorithms applied to the flooding graph. In the case of \cite{meyer91} or
\cite{waterfallsbs}, by the scheduling of the shortest distance algorithm
(ultrametric flooding distance for node weighted graphs or topographic
distance for node weighted graphs). These are only a few examples of
algorithms which may be used, among a large number of others.

\subsection{Arbitrary choices after choiceless prunings of the flooding graph}

We have seen how to reduce the number of flooding tracks and paths.\ The
flooding graph $G$ is transformed into an oriented graph $\overrightarrow
{G}.\ $The operator $\zeta$ prunes further this graph, again without arbitrary
choices.\ After $k$ iterations, only the flooding paths or flooding tracks of
the graph $G$ with a lexicographic depth $k+1$ remain. We then "transport" the
arrows of the graph $\zeta^{(k)}\overrightarrow{G}$ onto the flooding graph
$G$ and get a graph $G^{(k)}:$ we keep an edge $(p,q)$ of $G$ if and only if
there exists an arrow $\overrightarrow{pq}$ or $\overrightarrow{qp}$ in the
graph $\zeta^{(k)}\overrightarrow{G}.\ $Like that we obtain a flooding graph
$G$ in which the only remaining flooding tracks and paths have a steepness
$\geq k+1.$ A node $p$ belongs to two catchment basins, if there exists two
flooding paths towards the corresponding minima, with identical $k+1$ first
edges forming a track of $k+1$ maximal steepness. Fig.\ref{wshdex11} presents
4 pruning states of the node weighted graph in fig.\ref{wshdex10}%
A.\ Fig.\ref{wshdex11}A presents the associated flooding graph (without the
loops on the isolated regional minima), where the edge weights have been
added. Two types of methods may be applied to this graph.\ 

\textbf{Method 2: }After $k$ steps of pruning, the graph $G^{(k)}$ has only
flooding paths and chains of steepness $\geq k+1.\ $Again any algorithm
developed for node or edge weighted graph may be applied to this graph.\ But
as the remaining flooding paths are extremely scarce, even the loosest and
most myopic algorithms will do a good job. In particular the algorithm by
B.Marcotegui which normally selects flooding paths of steepness $\geq1$ now
selects paths of steepness $\geq k+1,$ as they are the only available.\ The
same is true for the algorithm based on the topographic distance, selecting
flooding paths of steepness $\geq2.$ The remaining arbitrary choice, if any,
is taken in charge by the scheduling of the shortest distance algorithm
(ultrametric flooding distance for node weighted graphs or topographic
distance for node weighted graphs).%

\begin{figure}
[ptb]
\begin{center}
\includegraphics[
height=2.0332in,
width=4.6534in
]%
{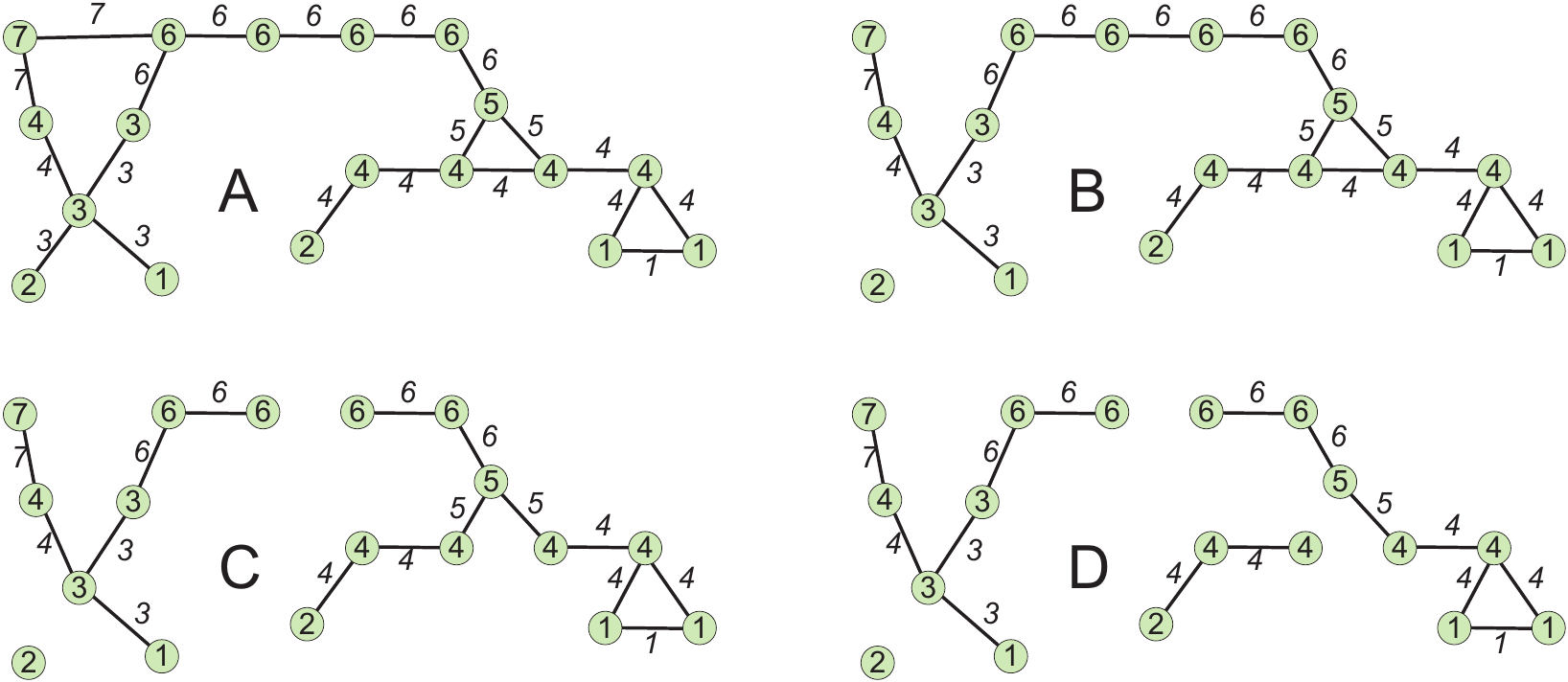}%
\caption{Four graphs with flooding paths of increasing steepness and
decreasing overlapping zones of the catchment basins.}%
\label{wshdex11}%
\end{center}
\end{figure}

\textbf{Method 3: }We continue the pruning $\zeta^{(k)}\overrightarrow{G}$
until there remain no arrows which are head to tail, except in the regional
minima.\ This is the case in fig.\ref{wshdex10}F.\ There exists a node with
weight $4$ which is the origin of two arrows with weight $4.\ $If we
arbitrarily suppress one of them, we are also done and have a partition. There
are two solutions possible at this stage of pruning. Thus if a node is the
origin of several arrows, one leaves only one.\ (this method has often been
used in hardware implementations, but without preliminary pruning
\cite{Bieniek2000907},\cite{lemonth}).

\subsection{Maximal prunings with scarcely needed choices}

For $k=\infty,$ we obtain a maximal pruning of the flooding graph.\ In fact we
may stop as soon the graph is cut into a number of components containing each
one regional minimum.\ The node and edge weights remain identical as in the
graph $G,$ but there remains only a minimal number of flooding paths and
chains. A node will be linked to two distinct minima by two flooding chains
only if there exists two paths with exactly the same weight distributions
towards these two minima ; this will rarely happen in natural images. If the
minima have distinct weights, each node is linked with one and only one
minimum by a flooding path of maximal steepness.

\textbf{Method 4: }We continue the pruning $\zeta^{(k)}\overrightarrow{G}$
until each connected component contains only one regional minimum and we are
done.\ If this cannot be achieved, then we resort to method 3.\ After an
infinite number of prunings, $\zeta^{(\infty)}\overrightarrow{G}$ does not
contain any head and tail arrows and method 3 can be applied.

\textbf{Method 5: }We slightly change the weights of the minima so that they
are all distinct.\ We continue the pruning $\zeta^{(k)}\overrightarrow{G}$
until each connected component contains only one regional minimum and we are
done.\ We know that this will happen for $k<\infty.$

\section{Conclusion}

We have established that node and weighted graphs represent the same
topography, with the same minima and the same catchment basins. We have
presented a method to reduce the overlapping zones of the catchment basins
without arbitrary choices. We finally have presented how to introduce such
arbitrary choices if needed.


\begin{thebibliography}{99}                                                                                               %


\bibitem {berge85}C.~Berge. \newblock {\em Graphs}. \newblock Amsterdam: North
Holland, 1985.

\bibitem {beucher79}S.~Beucher and C.~Lantu{\'{e}}joul. \newblock Use of
watersheds in contour detection. \newblock In \emph{Proc. Int. Workshop Image
Processing, Real-Time Edge and Motion Detection/Estimation}, 1979.

\bibitem {Bieniek2000907}Moga~A. Bieniek, A. \newblock An efficient watershed
algorithm based on connected components. \newblock {\em Pattern Recognition},
33(6):907--916, 2000. \newblock cited By (since 1996) 78.

\bibitem {waterfalls94}S.~Beucher. \newblock Watershed, hierarchical
segmentation and waterfall algorithm.
\newblock {\em ISMM94 : Mathematical Morphology and its applications to Signal
Processing}, pages 69--76, 1994.

\bibitem {Coustywshedcut}Jean Cousty, Gilles Bertrand, Laurent Najman, and
Michel Couprie. \newblock Watershed cuts: Minimum spanning forests and the
drop of water principle. \newblock {\em IEEE Transactions on Pattern
Analysis and Machine Intelligence}, 31:1362--1374, 2009.

\bibitem {gondranminoux}M.~Gondran and M.~Minoux.
\newblock {\em Graphes et Algorithmes}. \newblock Eyrolles, 1995.

\bibitem {lemonth}F.~Lemonnier.
\newblock {\em Architecture Electronique D\'edi\'ee aux Algorithmes Rapides de
Segmentation Bas\'es sur la Morphologie Math\'ematique}. \newblock PhD thesis,
E.N.S. des Mines de Paris, 1996.

\bibitem {waterfallsbs}B.~Marcotegui and S.~Beucher. \newblock Fast
implementation of waterfalls based on graphs.
\newblock {\em ISMM05 : Mathematical Morphology and its applications to Signal
Processing}, pages 177--186, 2005.

\bibitem {meyer91}F.~Meyer. \newblock Un algorithme optimal de ligne de
partage des eaux. \newblock In \emph{Proceedings $8^{\underline{\grave{e}me}}$
Congr\`{e}s AFCET, Lyon-Villeurbanne}, pages 847--857, 1991.

\bibitem {meyer94}F.~Meyer. \newblock Topographic distance and watershed
lines. \newblock {\em Signal Processing}, pages 113--125, 1994.

\bibitem {Najman199499}Laurent Najman and Michel Schmitt. \newblock Watershed
of a continuous function. \newblock {\em Signal Processing}, 38(1):99 -- 112,
1994. \newblock Mathematical Morphology and its Applications to Signal Processing.

\bibitem {mm2012}L. Najman and h. Talbot (ed) \newblock Mathematical
morphology \newblock {\em Wiley editor}, 2012

\bibitem {Roerdink01thewatershed}Jos B. T.~M. Roerdink and Arnold Meijster.
\newblock The watershed transform: Definitions, algorithms and parallelization
strategies. \newblock {\em Fundamenta Informaticae}, 41:187--228, 2001.
\end{thebibliography}
\end{document}